\documentclass[final]{IEEEtran}

\usepackage{times,graphicx, float, amssymb}
\usepackage{algpseudocode}
\usepackage{verbatim}
\usepackage{amsthm}
\usepackage{backref}
\usepackage{amsmath}
\usepackage{booktabs}
\usepackage{wrapfig}
\usepackage{caption}
\usepackage{subcaption}

\usepackage{amsmath,amsfonts,bm}

\def\eqref#1{equation~\ref{#1}}

\def\1{\bm{1}}

\def\mE{{\bm{E}}}

\def\mM{{\bm{M}}}

\def\mW{{\bm{W}}}

\DeclareMathAlphabet{\mathsfit}{\encodingdefault}{\sfdefault}{m}{sl}
\SetMathAlphabet{\mathsfit}{bold}{\encodingdefault}{\sfdefault}{bx}{n}

\newcommand{\E}{\mathbb{E}}

\newcommand{\R}{\mathbb{R}}

\DeclareMathOperator{\Tr}{Tr}

\newtheorem{theorem}{Theorem}
\newtheorem{corollary}{Corollary}[theorem]

\newtheorem{definition}{Definition}
\usepackage{hyperref}
\usepackage{url}

\newtheorem{example}{Example}[section]
\newtheorem{lemma}[theorem]{Lemma}
\usepackage{algorithm}
\newcommand{\Lap}{\mathbf L}
\newcommand{\GFMMD}{\mathcal{G\!F\!M\!M\!D}}

\title{Graph Fourier MMD for Signals on Graphs}
\author{Samuel Leone$^{\: 1}$ \thanks{1. Applied Math Program, New Haven, Connecticut} \and Aarthi Venkat$^{\:2}$ \thanks{2. Computational Biology \& Bioinformatics Program, Yale University, New Haven, Connecticut} \and Guillaume Huguet$^{\:3}$ \thanks{3. Dept. of Math. \& Stat, 
Université de Montréal ;
Mila - Quebec AI Institute, 
Montreal, QC, Canada} \and Alexander Tong$^{\:4}$ \thanks{4. Dept. of Comp. Sci. \& Oper. Res., 
Université de Montréal ; 
Mila - Quebec AI Institute, Montreal, QC, Canada} \and Guy Wolf$^{\:5}$ \thanks{5. Dept. of Math. \& Stat, 
Université de Montréal ;
Mila - Quebec AI Institute, 
Montreal, QC, Canada} \and Smita Krishnaswamy$^{\:6}$ \thanks{6. Department of Computer Science, 
Department of Genetics, 
Yale University, 
New Haven, Connecticut. Correspondence to: smita.krishnaswamy@gmail.com}}

\begin{document}
%
\maketitle
\begin{abstract}
  While numerous methods have been proposed for computing distances between probability distributions in Euclidean space, relatively little attention has been given to computing such distances for distributions on graphs. However, there has been a marked increase in data that either lies on graph (such as protein interaction networks) or can be modeled as a graph (single cell data), particularly in the biomedical sciences. Thus, it becomes important to find ways to compare signals defined on such graphs. Here, we propose Graph Fourier MMD (GFMMD), a novel distance between distributions and signals on graphs. GFMMD is defined via an optimal witness function that is both smooth on the graph and maximizes difference in expectation between the pair of distributions on the graph. We find an analytical solution to this optimization problem as well as an embedding of distributions that results from this method.  We also prove several properties of this method including scale invariance and applicability to disconnected graphs. We showcase it on graph benchmark datasets as well on single cell RNA-sequencing data analysis. In the latter, we use the GFMMD-based gene embeddings to find meaningful gene clusters. We also propose a novel type of score for gene selection called {\em gene localization score} which helps select genes for cellular state space characterization.
\end{abstract}

\section{Introduction}

Here, we address the question of how to organize and compare signals on graphs in such a way that accounts for geometric structure on their underlying space. In particular, given a weighted graph $\mathcal G = \mathcal (V, \mathcal E, w)$ and a set of functions $\{f_i\}_i$ on the vertices: $f_i : \mathcal V \to \mathbb R$, how can we structure and analyze these signals? We will first consider the case when $f_i$ is a probability mass function and extend the framework to arbitrary signals. This has a very natural applications to many modern datasets. 



 We present a new distance that belongs to the family of integral probability metrics \cite{sriperumbudur2012empirical} called maximal mean discrepancy or MMD. Integral probability metrics are distances between probability distributions that are characterized by a witness function that maximizes the discrepancy between distributions in expectation. MMDs have further structure in the witness function, requiring that they come from a Reproducing Kernel Hilbert Space.
Our notion of MMD, that we call \emph{Graph Fourier MMD} (GFMMD), is a distance between signals on a data graph that is found by analytically solving for an optimal witness function. GFMMD borrows notions from optimal transport, but does not require a distance metric, and thus generalizes to any undirected graph with nonnegative affinities. Furthermore, through the use of Chebyshev polynomials ~\cite{mason2002chebyshev}, GFMMD can be computed rapidly, and has a closed-form solution. We demonstrate its potential on toy datasets as well as single cell data, where we use it to identify gene modules. In the single cell setting, we focus on the application of embedding a set of genes on a graph of cells, as created from single cell RNA-sequencing data, and also in measuring whether the expression of a gene is localized (i.e., characteristic of a subpopulation of cells) or global like a house-keeping gene.

Our main contributions are as follows: 1) We define Graph Fourier MMD as a distance between signals on arbitrary graphs, and prove that it is both an integrable probability metric and maximum mean discrepancy. 2) We derive an exact analytical solution for GFMMD which can be approximated in $\mathcal O(n(\log n + m^2))$ time to calculate all pairwise-distances between distributions, where $n$ is the number of vertices of the graph and $m$ is the number of signals. 3) We derive a feature map for GFMMD that allows for efficient embeddings and dimensionality reduction. 4)  We provide an efficient Chebyshev approximation method for computing GFMMD among a set of signals. 5) We showcase application of GFMMD to single cell RNA-sequencing data. In short, Graph Fourier MMD is a simple, interpretable, and above all else effective method of comparing abstract distributions.







\subsection{Preliminaries}

\paragraph{Integral probability metrics} IPMs~\cite{muller1997integral,sriperumbudur2012empirical} constitute a family of distances between probability distributions. They are often used when dealing with empirical samples (datasets) sampled from a continuous space. In contrast, the alternative class of $\phi$-divergences (such as KL-divergence) is often less useful as a measure between empirical samples with poor behavior when the domains do not overlap. In contrast to $\phi$-divergences, integral probability metrics are defined over a metric space, this allows for a reasonable distance between distributions with non-overlapping support. 


\begin{definition}\label{def:IPM}
Given a metric space $(\mathcal X, d)$, a family $\mathcal F$ of measurable, bounded functions on $\mathcal X$, and two measures $P$ and $Q$ on $\mathcal X$, the IPM between $P$ and $Q$ is defined as
\[ \gamma_{\mathcal F}(P,Q) \triangleq \sup_{f \in \mathcal F}\mathbb E_P(f) - \mathbb E_Q(f). \]
\vspace{-5mm}
\end{definition}

Here, $\mathcal F$ is a family of ``witness function'' since it emphasizes the differences between $P$ and $Q$, choosing a certain $\mathcal F$ determines the IPMs. For certain classes of $\mathcal F$, the resulting distance is called an \emph{kernel Maximum Mean Discrepancy} (MMD) ~\cite{gretton2012kernel}. If $\mathcal H$ is a Reproducing Kernel Hilbert Space (RKHS) of functions on $\mathcal X$ (equipped with norm $\|(\cdot)\|_{\mathcal H}$), then the IPM corresponding to $\mathcal F = \{f : \| f \|_{\mathcal H} \leq 1\}$ is an MMD. Numerous distances between distributions are IPMs, given a suitable choice of $\mathcal F$. For example, the Wasserstein distance is an IPM where $\mathcal F$ corresponds to the family of Lipschitz functions.

\paragraph{Wasserstein Distance} The Earth Mover's Distance (EMD), also known as the 1-Wasserstein distance, is a distance between probability distributions designed to measure the least amount of "work" it takes to move mass from one distribution to another. Formally, we are given two distributions $P$ and $Q$ on a measure space $(\Omega, \mathcal F, \mu)$ and a distance $d : \mathcal X \times \mathcal X \to \mathbb R$. Most commonly, $\Omega$ might be a Riemann manifold, $\mathbb R^d$, or in our case, a finite graph. We define the space of couplings of $P$ and $Q$, denoted $\Pi(P,Q)$ to be the set of joint probability distributions whose marginals are equal to $P$ and $Q$. 

\begin{definition}\label{def:Wasserstein}
    The 1-Wasserstein Distance between $P$ and $Q$ is defined to be: 
    
\vspace{-2mm}    
\[ W(P,Q) \triangleq \min_{\pi \in \Pi(P,Q)} \int_{\mathcal X \times \mathcal X} d(x,y)\pi(dx,dy). \]
\vspace{-2mm}

\end{definition}

The supremum joint distribution $\pi$ would then be called the \emph{optimal transport plan}. In the case that $\Omega$ is finite (say of size $n$), $\Pi(P,Q)$ could be thought of as the set of $n \times n$ matrices $\pi$ for which $\pi \mathbf{1} = P, \mathbf{1}^T\pi = Q$. Then we could represent distances in a $n \times n$ matrix $D$, and the EMD is given by $\min_\pi \pi \cdot D$.  Typical solutions to EMD in its {\em primal form} are found using linear programming. The Kantorivich-Rubinstein Theorem, however, provides a dual formulation in terms of smooth functions:

\begin{theorem}\label{thm:K-R}(Kantorovich-Rubinstein)
The EMD is an IPM with $\mathcal F$ the space of $1$-Lipschitz functions
    \vspace{-2mm}
    \[ W(P,Q) = \sup_{\|f\|_{\leq 1}} \mathbb E_P(f) - \mathbb E_Q(f). \]
    \vspace{-5mm}
\end{theorem}

We refer to \cite{dudley2018real} for a proof of the previous theorem. Intuitively, we can think of suitable functions $f$ as being varying slowly over $\mathcal X$. The 1-Lipschitz constraint prevents witness functions from behaving too erratically over the space. As we will see, duality provides a valuable intuition for using smooth functions to compare functionals in abstract spaces.

\paragraph{The Graph Laplacian}

For a weighted graph $\mathcal G = (\mathcal V, \mathcal E, w)$ on $n$ vertices, we have a number of associated matrices. The first of which is an adjacency / affinity matrix $\mathbf{A}$ for which, given vertices $a$ and $b$, $\mathbf{A}(a,b) = w(a,b)$; for our purposes, we assume $w(a,b) \geq 0$. In the case when $\mathcal V$ belongs to a metric space $(\mathcal X, d)$, we have an associated distance matrix $M$ for which $M(a,b) = d(a,b)$ for all $a,b \in \mathcal V$. Oftentimes, the affinity matrix $\mathbf{A}$ is generated by a nonlinear kernel function $k(\cdot)$ so that 
$\mathbf{A}(a,b) = k(M(a,b))$. For our purposes, if $\mathbf{A}$ is generated in this way, we will call $\mathcal G$ a \emph{affinity graph}. 
There is also a diagonal degree matrix for which $\mathbf{D}(a,a) = \sum_{b \in \mathcal V}w(a,b)$. Finally, we define the combinatorial Laplacian $\Lap = \mathbf{D} - \mathbf{A}$. It can be shown that for any function on the vertices $f$, $f^T\Lap f = \sum_{(a,b) \in \mathcal E} w(a,b)(f(a) - f(b))^2$. From this, it's clear that $\Lap$ is positive semi-definite, and thus has a spectrum $\{(\lambda_1, \psi_1)... (\lambda_n, \psi_n)\}$, where $\lambda_1 \geq ... \geq \lambda_n$.

\paragraph{Effective Resistances} In \cite{spielman2008graph}, graphs are regarded as electrical circuits with edge weights providing capacities. In such a graph, the \emph{effective resistance} Re($a,b$) between vertices $a$ and $b$ is equal to Re($a,b) = \|\Lap^{-\frac{1}{2}}\delta_a - \Lap^{-\frac{1}{2}}\delta_b\|_2^2$, where $\delta_a,\delta_b$ are the one-hot encodings of vertices $a$ and $b$, respectively. Effective resistances provide valuable information about a graph. For instance, Spielman \& Srivastava use effective resistances between adjacent vertices to sparsify a graph. We will show, if we view $\delta_a,\delta_b$ as probability densities concentrated at $a,b$, Graph Fourier MMD provides an extension of resistances to arbitrary probability distributions on graphs. Effective resistences have been shown to be related to commute times, thus this provides a generalization of commute time when the initial and final position are not localized to a single node \cite{chandra1989electrical}.  


\subsection{Related Work}


The closest related work is that of \cite{verma2017hunt}, which constructs a family of spectral distances between graph signals based on weighted Fourier transforms. Algebraically, our distance $\GFMMD$ resembles a special case of these distances. Another similar distance is Diffusion EMD \cite{tong2021diffusion}, which involves diffusion graph signals to different scales using a diffusion operator (similar to that of a diffusion map \cite{coifman2006diffusion}) to create multiscale density estimates of the data. Then Diffusion EMD computes weighted $L^1$ distance between the multiscale density estimates of different signals. While this method is faster than most primal methods for EMD computation, it can be inaccurate unless the graph is significantly large. 

In \cite{le2022sobolev, le2019tree, essid2018quadratically}, the authors consider the EMD between distributions defined on a \emph{distance} graph, that is the edge weights define the cost of moving mass from one node to another. The authors in \cite{le2022sobolev, le2019tree} provide a closed-form solution that relies on a graph shortest path distance. In this setting, there is no sparse approximation to diffusion distances in terms of graph shortest path. We consider a different problem where the edges of the graph are affinities. Among methods for MMD, the most common method has been a sampling based method that also forms a 2-sample Kernel test based on defining a kernel between empirical observations \cite{gretton2012kernel}.  Note that semantically this takes distances between point clouds themselves by modeling them as a data graph with vertices as points. We define a method of taking signals which generalizes to an arbitrary graph, on a point cloud or otherwise, and demonstrate its effectiveness both when the graph lies in a metric space and when adjacencies are binary.

\section{Methods}

Given two probability distributions $P$ and $Q$ on an arbitrary graph, we are interested in taking a meaningful distance between them in a way that incorporates graph structure. Unlike the distance setting, there is no obvious notion of Lipschitzness. However, there is still a notion of smoothness. Indeed, we define Graph Fourier MMD as the MMD induced by witness functions $f$ which are \emph{smooth over the graph}. That is, they have a low value in $f^T \Lap f$. 

\begin{definition}\label{def:GFMMD_def}
    Let $\mathcal G = (\mathcal V, \mathcal E, w)$ be a finite  graph with Laplacian  $\mathbf{L}$ and $P,Q$ be two bounded probability distributions on $\mathcal V$. The Graph Fourier MMD between $P$ and $Q$ is
    \vspace{0mm}
    \[ \GFMMD(P,Q)\triangleq \max_{f : f^T\mathbf{L}f \leq 1}\mathbb E_{P}(f) - \mathbb E_Q(f). \] 
    \vspace{-5mm}
\end{definition}

Note that this definition holds for any construction of a positive semi-definite Laplacian matrix $\mathbf{L}$. We can show that, under reasonable conditions, GFMMD is finite and simple to compute. First, we need to establish a property which indicates that $P$ and $Q$ do not differ on the scale of connected components. 

\begin{figure}
    \centering
    \includegraphics[width=0.6 \linewidth]{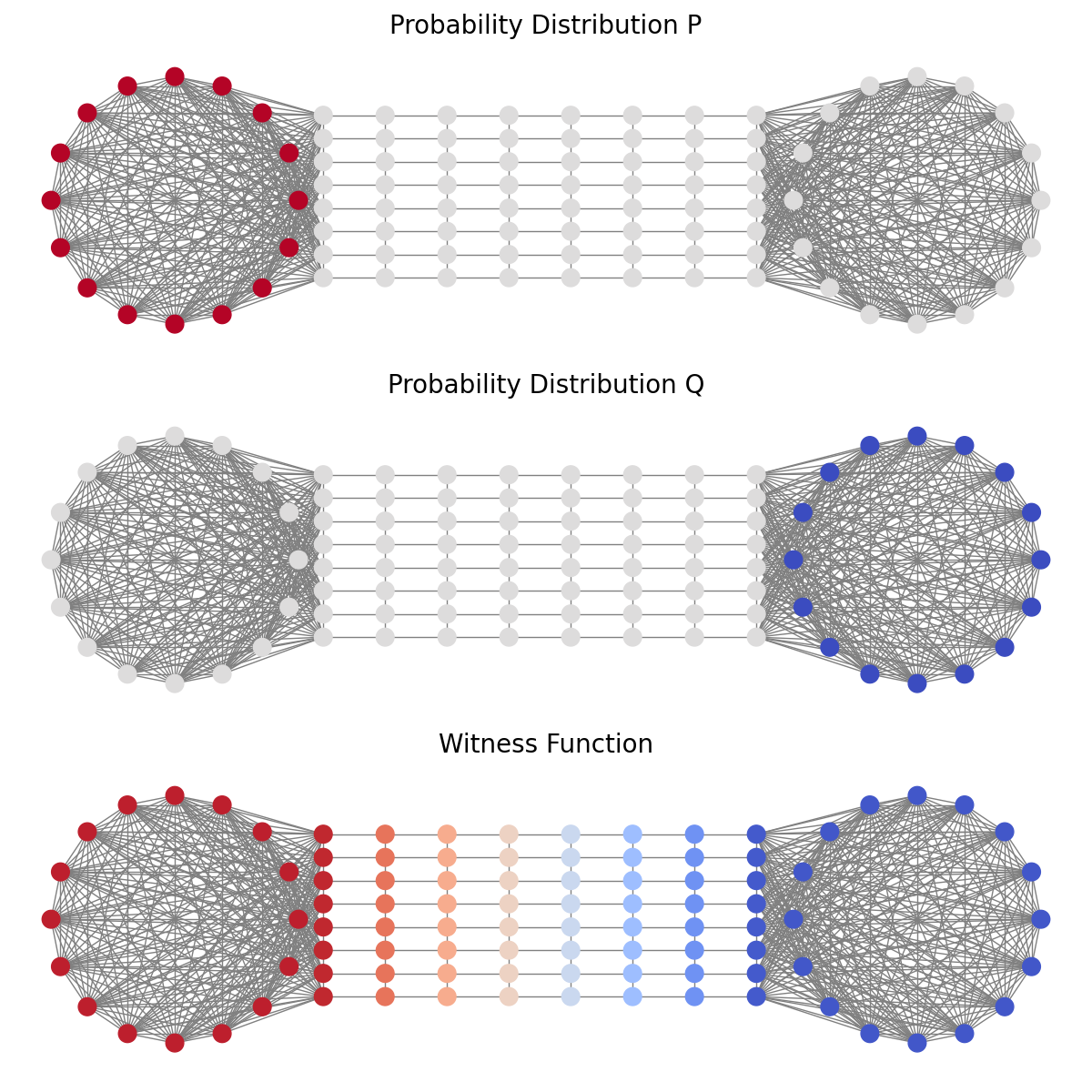}
    \caption{The witness function $f$ is positively activated over the left cluster, negatively activated over the right cluster, and varies smoothly along the middle segment. $f$ attests to where $P$ and $Q$ differ, but smoothness requires that it do so while incorporating graph structure.}
    \vspace{-5mm}
\end{figure}

\begin{definition} 
    Two probability distributions $P$ and $Q$ are said to have equal component mass if, for all connected components $S$ of $\mathcal G$, 
    $P(v \in S) = Q(v \in S)$.
\end{definition}

\begin{theorem}  \label{thm:gfmmd_solution}
    Let $P$ and $Q$ be bounded probability distributions defined on $\mathcal V$. If $P$ and $Q$ have equal component mass, then $\GFMMD(P,Q) =  \|\mathbf{L}^{-\frac{1}{2}}(P - Q)\|_2$. And otherwise, $\GFMMD(P,Q) = +\infty$.
\end{theorem} 

\begin{proof} 

    Suppose first that $P$ and $Q$ do not have the equal mass property. Then there exists a connected component $S$ for which, 
    \[ \sum_{v \in S} P(v) < \sum_{v \in S} Q(v) \]
    In particular, we can write $\sum_{v \in S} P(v) = \sum_{v \in S} Q(v) - c$ for some $c > 0$. Now, let $f_\alpha$ be a signal such that $f_\alpha(v) = \alpha$ if $v \in S$ and $f_\alpha(v) = 0$ otherwise. Then certainly, $f_\alpha^T \Lap f_\alpha = 0$, since it is known that indicator functions for connected components arer in the null space of $\Lap$. And so $f_\alpha^T \Lap f_\alpha \leq 1$, yet, 
    \begin{align*} \E_P(f_\alpha) - \E_Q(f_\alpha) &= \sum_{v \in \mathcal V} P(v)f_\alpha(v) -
    \sum_{v \in \mathcal V} Q(v)f_\alpha(v) \\ 
    & = \sum_{v \in S} \alpha P(v) - 
    \sum_{v \in S} \alpha Q(v)\\ & = \alpha c  \end{align*}
    S GFMMD is defined as $\sup_{f^T \Lap f \leq 1} \E_P(f) - \E_Q(f)$, we have $\GFMMD(P,Q) \geq \alpha c$. Taking $\alpha \to \infty$, we have $\GFMMD(P,Q) = +\infty$. \\ 

    Now suppose that $P$ and $Q$ do have the equal mass property. If we let 
    $\mathbb I_{S_1}... \mathbb I_{S_m}$ be indicator functions for connected components $S_1... S_m$, the equal mass property insists that 
    $P^T \mathbb I_{S_i} = Q^T\mathbb I_{S_i}$ for all $i$. And thus, $(P-Q)^T \mathbb I_{S_i} = 0$. Since it is known that these indicator functions form a basis for the kernel of $\Lap$, it follows that $P-Q \in \text{ker}(\Lap)^\perp$. Now, any function $f$ such that 
    $f^T \Lap f$ can be broken up into $f = f_1 + f_2$, where $f_1 \in \text{ker}(\Lap)$ and $f_2 \in \text{ker}(\Lap)$. Finally, observe that we can view $P$ and $Q$ as probability vectors indexed over $\mathcal V$. And so, 
    \begin{align*} \E_P(f) - \E_Q(f) &= P^Tf - Q^Tf = (P-Q)^T(f_1 + f_2) \\
    &  = (P-Q)^Tf_1 + (P-Q)^Tf_2 \\ & = (P-Q)^Tf_2 \end{align*}
    Furthermore, $f^T \Lap f = f_2^T \Lap f_2$. Combined, these observations tell us that we may assume, without loss of generality, that $f \in \text{ker}(\Lap)^\perp$. And thus, 
    \[ \GFMMD(P,Q) = \sup_{\substack{f^T\Lap f \leq 1 \\ f \in \text{ker}(\Lap)^\perp}} (P-Q)^T f \]
    Now, for any such $f$, we can define $y = \Lap^{\frac{1}{2}}f$. And thus, 
    $f^T \Lap f = f^T \Lap^{\frac{1}{2}} \Lap^{\frac{1}{2}} f = \|y\|_2^2$. Furthermore, since $f \in \text{ker}(\Lap)^\perp$, $f = \Lap^{-\frac{1}{2}}y$. Here, $\Lap^{-\frac{1}{2}}$ is the square root of the Moore-Penrose pseudoinverse $\Lap^\dagger$ of $\Lap$. Thus, 
    \[ \GFMMD(P,Q) = \sup_{y : \|y\|_2^2 \leq 1} (P-Q)^T \Lap^{-\frac{1}{2}}y \]
    Which clearly, by Cauchy Schwarz, is simply equal to $\|(P-Q)^T \Lap^{-\frac{1}{2}}\|_2^2  = \|\Lap^{-\frac{1}{2}}(P-Q)\|^2$, as desired.
\end{proof}

The effect of equal component mass is highly intuitive: we would expect ``infinite effort'' to move a probability distribution between disconnected sets of vertices. GFMMD possesses a set of convenient properties. Namely, we have a representation in terms of an explicit feature map $\mathbf{L}^{-\frac{1}{2}}$. So to compute pairwise distances, it is sufficient to apply the feature map and then take Euclidean distances. Also, the distance value is a true distance (particularly an MMD).

\begin{lemma}
    \label{lem:valid_distances}
    (i) $\GFMMD(\cdot, \cdot)$ defines a valid distance on the probability distributions acting on $\mathcal V$. Furthermore, (ii) $\GFMMD(P,Q)$ is a Maximum Mean Discrepancy with explicit feature map $\mathbf{L}^{-\frac{1}{2}}$ in the finite case.
\end{lemma}

\begin{proof}
    Let $\mathcal P$ denote the set of probability distributions on $\mathcal{V}$. Nonnegativity and symmetry of $\GFMMD$ are simple to verify directly from the definition. We now show the triangle inequality. Fix $P,Q,R \in \mathcal P$. Letting $\mathcal F = \{f \: : \: f^T \Lap f \leq 1\}$,
    \begin{align*} \GFMMD(P,Q) &= \sup_{f \in \mathcal F}\E_P(f) - \E_Q(f) \\ 
    & = \sup_{f \in \mathcal F} \E_P(f) - \E_R(f) + \E_R(f) - \E_Q(f) \\ & \leq \sup_{f \in \mathcal F} \E_P(f) - \E_R(f) + \sup_{f \in \mathcal F}  \E_R(f) - \E_Q(f) \\ &= \GFMMD(P,R) + \GFMMD(Q,R) \end{align*}

We conclude (i) by showingd that $\GFMMD(P,Q) = 0$ if and only if $P = Q$. The first direction is trivial, since for all $f \in \mathcal F$, $\E_P(f) - \E_Q(f) = 0$. For the other direction, suppose that $\GFMMD(P,Q) = 0$. Then by theorem ~\ref{thm:gfmmd_solution}, $P$ and $Q$ necessarily have equal component mass. From this, it follows that for all components $S$, $\mathbb I_S^TP = \mathbb I_S^TQ$, so $P=Q$ over their projections onto $\ker(\Lap)$. But since $\GFMMD(P,Q) = \|\Lap^{-\frac{1}{2}}P - \Lap^{-\frac{1}{2}}Q\|_2 = 0$, $P = Q$ over $\ker(\Lap)^\perp$. Thus, $P = Q$.

For (ii), we see that $\GFMMD$ clearly takes the form of an MMD, so there is nothing to verify. And long as $\GFMMD(P,Q) < +\infty$, the distance is simply the Euclidean distance between $\varphi(P)$ and $\varphi(Q)$ using the feature map $\varphi: x \mapsto  \Lap^{-\frac{1}{2}}x$. 
\end{proof}

\subsection{Relationship with $f$-Spectral Distances}
The $f$ spectral distance between two signals $P$ and $Q$ per \cite{verma2017hunt} is defined as $\sum_{i} f(\lambda_i) (\hat{P}(i) - \hat{Q}(i))^2$, where $f$ is some monotone increasing or decreasing function, and $\hat{P}, \hat{Q}$ denotes the Fourier transform of $P$ and $Q$. Likewise, $\GFMMD(P,Q) = \sum_{ \lambda_i \neq 0} \frac{1}{\lambda_i}(\hat{P}(i) - \hat{Q}(i))^2$. Algebraically, $\GFMMD$ resembles the $f$-spectral distance for $f(x) = 1/x$, although there is different treatment of components in $\ker(\Lap)$, which arises from the fact that we only consider probability distributions.

\subsection{Relationship with Resistive Embeddings} 

Graph Fourier MMD provides a natural extension of resistive embeddings through application of the same feature map. In fact, resistive embeddings can be viewed as a special case of Graph Fourier MMD for dirac distributions. On the other hand, theorem ~\ref{thm:gfmmd_solution} offers a variational characteristic of resistive embeddings, which could be used for lower bounds. Finally, it can be shown that Graph Fourier MMD can be related to the the expected effective resistance of $X \sim P$ and $Y \sim Q$. Not only does this provide an interpretation of GFMMD, but it illustrates that spectral sparsification algorithms such as \cite{spielman2008graph} preserve GFMMD.

\begin{theorem} ~\label{thm:eff_res_couplings} If $X \sim P$ and $Y \sim Q$, not necessarily independent, then $\GFMMD(P,Q)^2 \leq \mathbb E_{X,Y}[Re(X,Y)]$.
\end{theorem}

\subsection{Graph Fourier MMD for Signal Localization}

By fixing a probability distribution $P$ and considering its distance to the uniform distribution $U = \frac{1}{n}\mathbf{1}$, we obtain a measure of how much $P$ concentrates on the graph. For a high GFMMD to the uniform distribution, we say a signal is \emph{localized}, and otherwise we say it is \emph{dispersed}. Interestingly, localization simply corresponds to length in feature space.

\begin{definition}
The {\em Localization Score} $s(P)$ of a signal $P$ is defined as $\GFMMD(P,U) = \|\Lap^{-\frac{1}{2}}(P - U)\|_2 = \|\Lap^{-\frac{1}{2}}P\|_2$.
\end{definition}


\subsection{Computational Complexity and Speedup}


Computation of the pseudoinverse or $\Lap^{-\frac{1}{2}}$ is roughly $O(n^3)$. An alternate approach would be to simply calculate the solution to $\Lap y = P-Q$ via conjugate gradient descent. We elect to use Chebyshev polynomials to approximate $\Lap^{-\frac{1}{2}}P$ and $\Lap^{-\frac{1}{2}}Q$, then take Euclidean distances. This is a large improvement in the particular case where the graph $\mathcal{G}$ is sparse i.e.\ $|E| = O(n \log n)$. In such cases, we present an $O(n \log n)$ algorithm for the computation of GFMMD which is substantially faster than naive implementations based on a Chebyshev polynomial approximation of the filter in Algorithm~\ref{algorithm:GFMMDAlg} as well as a KNN kernel. The steps for an arbitrary graph are the same, but with $\mathbf{W}$ provided.

This gives a total fine grained time complexity of $\mathcal O((k_1 + t)n \log n + k_2 n m \log m)$ and a space complexity of $O(n \log n + mn)$ space. Here, $t$ is the order of the Chebyshev polynomial, $k_1$ is the threshold for number of nearest-neighbors in constructing $\mathcal G$, and $k_2$ is the number of nearest-distributions we'd like to calculate. More simply, for fixed Chebyshev order, and number of neighbors, the time to estimate distances between all distributions is $\mathcal O(n\log n + nm^2)$.

\begin{algorithm}[H]
    \begin{algorithmic}
        \State \textbf{Input:} A set of $n$ points $X \subseteq \mathbb R^d$, $m$ probability distributions $f_i : X \to \mathbb R$ in an $n \times m$ matrix $F$, and a kernel function $k : X \times X \to \mathbb R$
        \State \textbf{Output:} An $m \times n$ embedding matrix $\mE$ in which 
        $\|\mE_i - \mE_j\| = \GFMMD(f_i, f_j)$. $M$ and a distance matrix in which 
        $\mM_{ij} = \GFMMD(f_i, f_j)$ \\
        \State Create a thresholded K-Nearest Neighbor graph $\mathcal G$ over $X$ with $\mathcal O(n \log n)$ edges
        \State $\mW_{ij} \gets k(X_i, X_j)$ for all $(i,j) \in \mathcal E$ 
        \State $\mathbf{L} \gets \mathbf{D} - \mathbf{W}$
        \State $\mE_i \gets \mathbf{L}^{-\frac{1}{2}} f_i$ by exact or Chebyshev approximation of the filter $h(\lambda) = \lambda^{-\frac{1}{2}}$
        \State $\mM_{ij} \gets \|\mE_i - \mE_j\|$ for all $i,j \in [n]$
        \State return $\mM,\mE$
    \end{algorithmic}
    \label{algorithm:GFMMDAlg}
    \caption{GFMMD in Metric Space}
    \label{algorithm:GFMMDAlg}
\end{algorithm}

\section{Experimental Results \& Applications}

In our experiments, signals are always nonnegative and normalized to be interpreted as probability distributions. However, the metric induced by 
$\Lap^{-\frac{1}{2}}$ \emph{always} induces a valid seminorm on graph signals (but in particular, a norm on probability distributions).

\subsection{Identifying Distributions on the Swiss Roll}

In this experiment, we generate random point clouds centered at points on the swiss roll. More specifically, we sample $n = 100$ points on the swiss roll $x_1, x_2,\ldots, x_n$ and around each of these points, generate a point cloud $d_i$ of size $m = 100$ points from a multivariate normal distribution centered at $x_i$. The result is $nm$ points in $\mathbb R^{10}$. For each $i,j \in [n]$, we have a known geodesic distance between $x_i$ and $x_j$. Across the different measures, we can see how well the distance between the point clouds $d_i$ and $d_j$ compares to the geodesic distance between their corresponding centers $x_i$ and $x_j$. We compare the induced probability distributions on point clouds using: 1) computation of earth Mover's Distance between point clouds in ambient space, 2) Sinkhorn algorithm~\cite{cuturi2013sinkhorn}, 3) Diffusion EMD~\cite{tong2021diffusion}, 4)  Kernel MMD~\cite{gretton2012kernel} between all pairs $p_i, p_j$ via random sampling ($20$ points from each distribution with replacement), 5) Graph Fourier MMD between $p_i, p_j$, using both the exact calculation and approximation via Chebyshev polynomials. We then take correlation between the estimated nearest distributions and geodesic distance between centers. Results are shown in Table~\ref{table:swissroll}. As we see, Graph MMD outperforms all other methods in accuracy and speed. In the appendix, the corresponding feature maps are visualized alongside geodesic distance, where GFMMD visually outperforms other methods at extracting manifold nonlinearities.

\begin{table}[ht]
\caption{Comparison of runtime and Spearman-$\rho$ correlation to ground truth manifold distances between distributions with mean $\pm$ standard deviation over 10 seeds for 100 distributions of 100 points each on a swiss roll manifold. The exact Graph MMD is most performant but requires a eigen-decomposition. The Chebyshev approximated Graph MMD (Chebyshev, $t$) is extremely fast and almost as performant at even low orders $t$.}
\label{table:swissroll}
\centering
  \resizebox{\columnwidth}{!}{%
\begin{tabular}{lrrr}
\toprule
Method & Spearman-$\rho$ & 10-NN time (s) & All-pairs time(s) \\
\midrule
DiffusionEMD & 0.584 $\pm$ 0.017 & 2.171 $\pm$ 0.265 & 3.341 $\pm$ 0.333 \\
Exact & 0.253 $\pm$ 0.022 & 26.881 $\pm$ 1.104 & 26.881 $\pm$ 1.104 \\
Sinkhorn & 0.250 $\pm$ 0.022 & 54.346 $\pm$ 17.576 & 54.346 $\pm$ 17.576 \\
rbf-kernel-MMD & 0.509 $\pm$ 0.021 & 5.016 $\pm$ 0.237 & 5.016 $\pm$ 0.237 \\
Graph MMD (Exact) & \textbf{0.613 $\pm$ 0.019} & 139.453 $\pm$ 16.790 & 139.468 $\pm$ 16.794 \\
Graph MMD (Cheby, 8) & 0.606 $\pm$ 0.024 & \textbf{0.619 $\pm$ 0.057} & \textbf{0.641 $\pm$ 0.056} \\
Graph MMD (Cheby, 64) & 0.593 $\pm$ 0.021 & 1.155 $\pm$ 0.035 & 1.163 $\pm$ 0.035 \\
Graph MMD (Cheby, 512) & 0.612 $\pm$ 0.018 & 6.249 $\pm$ 2.896 & 6.258 $\pm$ 2.895 \\

Graph MMD (Cheby, 4096) & 0.612 $\pm$ 0.018 & 48.138 $\pm$ 1.184 & 48.159 $\pm$ 1.182 \\
\bottomrule
\end{tabular}
}
\vspace{-5mm}
\end{table}

\subsection{Single cell Analysis with GFMMD}
 
To demonstrate the utility of Graph Fourier MMD for biological analysis, we leverage publicly available single-cell RNA sequencing dataset of CD8-positive T cells \cite{Zheng2017-mt}. CD8-positive T cells are adaptive immune cells known to be critical for mediating immune response in infection, cancer, and other diseases. We apply Algorithm \ref{algorithm:GFMMDAlg} with the adaptive Gaussian Kernel \cite{Moon19} between datapoints, to compute GFMMD between genes, where each gene (of $1,991$ genes) is regarded as a distribution in a nearest neighbor cell graph over $9,167$ cells. In Figure~\ref{fig:geneclusters}A, we visualize the gene embedding using both PCA and PHATE \cite{Moon19}. We find that clusters $0-9$ in from the gene embedding show characteristic expression on the cellular embedding in Figure \ref{fig:geneclusters}B. In other words, the subplots in Figure \ref{fig:geneclusters}B represent a PHATE map of the cells in this dataset, and when we highlight the expression of {\em gene clusters} on the cells we see that these clusters have localized expression on the cellular manifold. To interpret these gene clusters for biological significance, we analyzed the gene set enrichment of clusters 6 and 7 with Enrichr~\cite{Chen2013-lq}, which show high expression in opposite ends of the cellular manifold (see Figure \ref{fig:geneclusters}. Enrichr shows that cluster 7 has strong enrichment for signatures of a naive T cell becoming activatied with mitosis and T cell activation signatures being significant. On the other hand, cluster 6 shows strong enrichment for an effector CD8 T cell, with signatures of cytotoxic activity and inflammatory signaling (interferon gamma). Thus, these genes can be used to characterize the cellular manifold as following a trajectory from naive to effector CD8 T cells. We compare these to gene clusters derived from DiffusionEMD, as well as to a more standard method of gene selection in biology: differential expression of genes in different areas of cellular state space based on a Wilcoxon rank sum test between the two manually curated cell clusters from \cite{Zheng2017-mt}. The gene clusters 4 and 8 from DiffusionEMD that were most enriched on the opposite ends of the manifold consisted of 6 genes and 11 genes, which resulted in no enrichment for the above signatures. These genes upregulated based on the Wilcoxon rank sum test give a much less clear picture of the cellular state space, with the same annotations scoring much lower. 

\begin{figure}{}
\centering
\includegraphics[width=\linewidth]{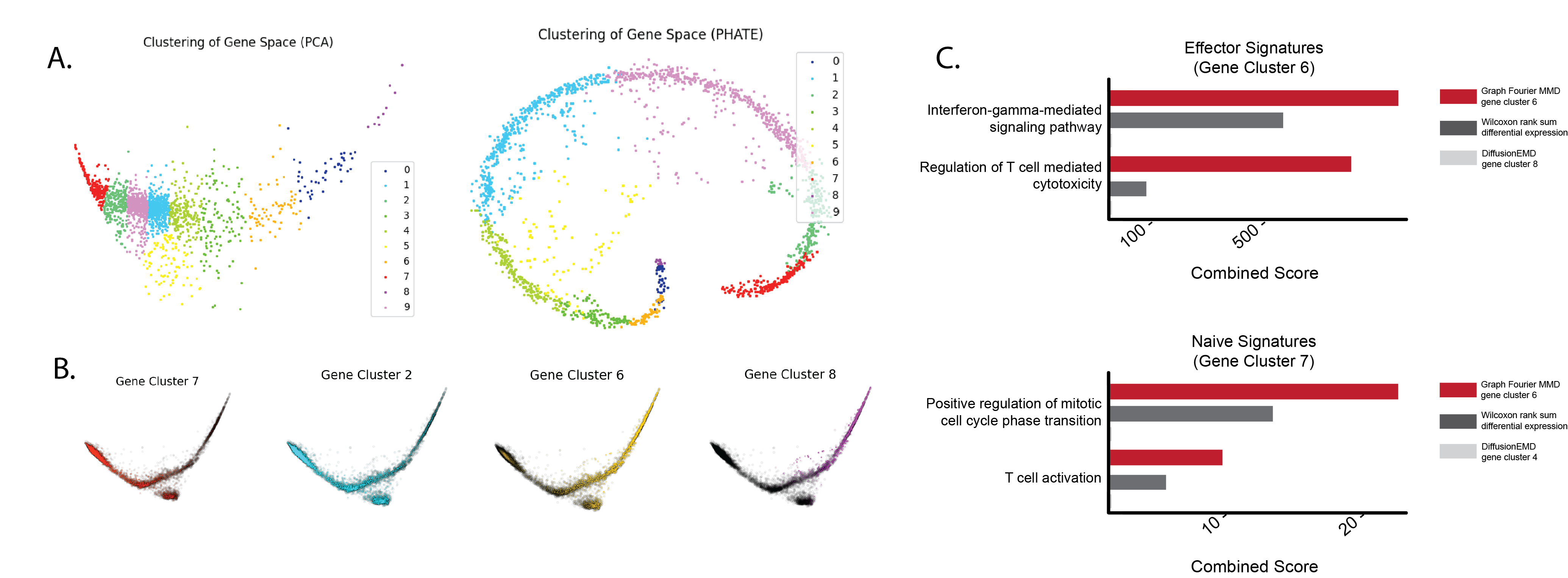}
\vspace{-3mm}
\caption{\small A. Embeddings of genes of the dataset from \cite{Zheng2017-mt} computed by the GFMMD Algorithm, visualized with PCA and PHATE \cite{Moon19}, colored by results of K-means clustering. B. Embeddings of cells from \cite{Zheng2017-mt} visualized with PHATE. Each plot is colored by the average expression of genes in the marked cluster over cells. C. Comparison of enrichment scores from Enrichr \cite{Chen2013-lq} on T-cell relevant annotations,  between GFMMD-based gene sets from clusters 7, 6 and differential expression-based gene sets.} 
\vspace{-5mm}
\label{fig:geneclusters}
\end{figure}

\paragraph{Local Genes} A novel type of analysis enabled by GFMMD is a search for {\em localized} signals. Often, researchers in the single cell field search highly variable genes, but we posit that genes that have localized expression on cellular manifolds can be used to characterize salient cellular subtypes. We propose the use of the localization score of the genes, viewed as probability distributions over the cells. Based on this score, in Figure~\ref{fig:localization}, we visualize first most local gene, 10th most local gene, and 20th most local. Here, we compare localization scores between housekeeping genes and the gene signature for naive CD8+ T cells. Housekeeping genes are expressed highly in many systems, but are not known to have a function that contributes strongly to cell-cell variation for T cells \cite{Eisenberg2003-jb, Wang2021-ga, De_Jonge2007-tx}. By contrast, cells enriched for the naive CD8+ T cell signature are a subset of T cells along the T cell differentiation axis. We show that the localization score is an order of magnitude higher for the naive gene signature versus the housekeeping signature Figure \ref{fig:localization} (see the appendix), validating our intuition about localized genes.

\vspace{-2mm}

\section{Conclusion}

In this paper we have introduced Graph Fourier MMD, a framework for taking distances between  signals on graphs and generating embeddings in which these distances hold. We have shown its intuitive performance in both the Riemannian and abstract graphical setting for known distributions, as well as its advantage in speed, and ability to capture global properties of the underlying data manifold compared to alternative methods like Earth Mover's Distance and Diffusion EMD. Its rapidity makes it particularly useful for high dimensional datasets, such as single cell data, where we have showed its ability to capture the natural trajectories of gene expression.



\bibliographystyle{plain}
\bibliography{main}

\newpage 
\appendix

\section{Appendix}

\subsection{Proof of Stated Results}

\subsubsection{Effective Resistances \& Couplings} 

\noindent \textbf{Theorem ~\ref{thm:eff_res_couplings}}\textit{ If $X \sim P$ and $Y \sim Q$, not necessarily independent, then $\GFMMD(P,Q)^2 \leq \mathbb E_{X,Y}[Re(X,Y)]$}

\begin{proof}
    The bias variance decomposition in dimension $n$ states that for a random vector $Z$ and point $a \in \mathbb R^n$, $\mathbb E\|Z - a\|^2 = \|\mathbb EZ - a\|^2 + Var(Z)$. Let $\varphi(a)$ denote 
    column $a$ of $\Lap^{-\frac{1}{2}}$, so that $\varphi(X), \varphi(Y)$ are random vectors. We have, $\GFMMD(P,Q)^2 = \|\Lap^{-\frac{1}{2}}P - \Lap^{-\frac{1}{2}}Q\|^2  = \|\mathbb E_X[\varphi(X)] - \mathbb E_Y[\varphi(Y)] \|^2  = \|\mathbb E_Y[\mathbb E_X[\varphi(X)]] - \mathbb E_Y[\mathbb E_X[\varphi(Y)]]\|^2$.
    By Fubini's Theorem for expectations, this is equal to,
    $\|\mathbb E_{X,Y}[\varphi(X)] - \mathbb E_{X,Y}[\varphi(Y)]\|^2  = \|\mathbb E_{X,Y}[\varphi(X)-\varphi(Y)]\|^2$. By the Bias-Variance Decomposition, 
    $\|\mathbb E_{X,Y}[\varphi(X)-\varphi(Y)]\|^2 = \mathbb E_{X,Y}\|\varphi(X) - \varphi(Y)\|^2 - \text{Var}_{X,Y}[\varphi(X) - \varphi(Y)]  \leq \mathbb E_{X,Y}\text{Re}(X,Y)$. We recognize that $\|\varphi(X) - \varphi(Y)\|^2 = \text{Re}(X,Y)$.
\end{proof}

\begin{corollary} ~\label{cor:spectral_bound}
    Suppose $P \: \& \: Q$ agree on a set of size $\mathcal A$, and suppose the union of  their supports is $\mathcal S$. Then, 
    $\textit{GFMMD}(P,Q) \leq \sqrt{(1-p)M} \leq \sqrt{(1-p)/2\lambda_2}$, where $p = \sum_{a \in \mathcal A}P(a)$, $M = \sup\{\text{Re}(X,Y) : X,Y \in S \setminus A \}$, and
    $\lambda_2$ is the Fiedler  value for the graph.
\end{corollary}
    
\begin{proof} 
    Let $Z$ be a Bernoulli random variable with success probability $p$. First, choose $X_0 \sim P, Y_0 \sim Q$. Construct $X = X_0 \mathbb I\{Z = 0\} + Z\mathbb I\{Z = 1\}$ and $Y = Y_0 \mathbb I\{Z = 0\} + Z\mathbb I\{Z = 1\}$. Thus, $\GFMMD(P,Q)^2 \leq \mathbb E_{X,Y}\text{Re}(X,Y) \leq E_{X,Y}[\text{Re}(X,Y) | Z = 1] \mathbb P(Z=1) + (1-p) \mathbb E_{X,Y}[\text{Re}(X,Y) | Z = 0]\mathbb P(Z=0) \leq (1-p) \mathbb E_{X,Y}[\sup\{\text{Re}(X,Y) : X,Y \in S \setminus A] = (1-p)M$. Furthermore, we can provide an upper bound for $M$. The Courant-Fisher theorem tells us that for nonzero $x$, $x^T \Lap x \leq \frac{1}{\lambda_2}\|x\|^2$, as $1/\lambda_2$ is the maximal eigenvector of $\Lap^{-1}$. Thus, letting $a \neq b$ be arbitrary vertices, we have that $\text{Re(a,b)} = (\delta_a - \delta_b)^T \Lap^{\dagger}(\delta_a - \delta_b) \leq 2/\lambda_2$. In particular, maximizing over all $a,b \in \mathcal S \setminus \mathcal A$, $M \leq 2/\lambda_2$.
\end{proof}

\subsubsection{An Additional Result}

We can also show that there is a nice correspondence for PCA on the space of dirac-distributions $\{\delta_i\}_i$ on the vertices, upon applying the feature map offered by GFMMD. In fact, the best $k$-dimensional representation (by multidimensional scaling) of the vertices will coincide almost exactly with Hall's Spectral Graph Drawing\cite{hall1970r}, which uses the first $k$ nontrivial eigenvectors to represent vertices using coordinates in $\mathbb R^k$. This is made formal by Theorem~\ref{thm:like_hall_drawing}.

\begin{theorem}
  \label{thm:like_hall_drawing} If $X = \{ \delta_i\}_{i \in \mathcal V}$ is a family of Kronecker-delta functions centered at each vertex of $\mathcal G$, then the $k$-dimensional embedding which best preserves the distances between signals in $X$ is equivalent up to rescaling to Hall's Spectral Graph Drawing of the Graph $\mathcal G$ in $k$-dimensions. 
\end{theorem}

\begin{proof} Note that $X$, the data matrix of Kronecker Deltas, is equal to $\mathbf{I}$, the $n$-dimensional identity. So $\sqrt{T}\mathbf{L}^{-\frac{1}{2}}X = \sqrt{T}\mathbf{L}^{-\frac{1}{2}}$, hence the best $k$-dimensional embedding of $\sqrt{T}\mathbf{L}^{-\frac{1}{2}}$ (respecting the $L^2$ norm between columns) will be equivalent to Principal Component Analysis (P.C.A.). Since $\mathbf{L}^{-\frac{1}{2}}\mathbf{1} = 0$, $\mathbf{L}^{-\frac{1}{2}}$'s columns are mean-centered, so its covariance matrix of $\sqrt{T}\mathbf{L}^{-\frac{1}{2}}$ is $\frac{T}{n}\mathbf{L}^{{-\frac{1}{2}}^T} \mathbf{L}^{-\frac{1}{2}} = \frac{T}{n}\mathbf{L}^-$.

Since its columns and rows are already mean centered. And thus P.C.A. will select the eigenvectors of $\mathbf{L}^-$ corresponding to the $k$th largest eigenvalues. Note that these are precisely given by $\psi_{1}, \psi_{2}.. \psi_{k}$ with associated eigenvalues in $\mathbf{L}^-$ given by $\lambda_1^{-1}\ldots \lambda_k^{-1}$. Letting $\Lambda_k = \text{diag}(\lambda_1^{-1/2}\ldots \lambda_k^{-1/2})$ and $\Psi_k = \begin{pmatrix} \psi_1 \: \ldots \psi_k \end{pmatrix}$, P.C.A. would embed $\sqrt{T}\mathbf{L}^{-\frac{1}{2}}$ as, 
\vspace{-3pt}
\[ \Psi_k^T \mathbf{L}^{-\frac{1}{2}} = \Psi_k^T \Psi \Lambda^{-\frac{1}{2}}\Psi^T \] 
\[ = \begin{pmatrix} \mathbf{I}_k & 
\mathbf{0}_{n-k}
\end{pmatrix}\Lambda^{-\frac{1}{2}}\Psi^T  \] 
\[ 
 = \begin{pmatrix} \mathbf{I}_k\Lambda_k & 
\mathbf{0}_{n-k} \end{pmatrix}\Psi^T = \Lambda_k \Psi_k^T. \]

So our embedding of distributions would be given by $\Lambda_k \Psi_k^T$. On the other hand, Hall's Spectral Graph Drawing would embed the graph $\mathcal G$ simply as $\Psi_k^T$, since it chooses the first $k$ nontrivial eigenvectors of $\mathbf{L}$. Thus, coordinates in each embedding are the same up to the rescaling by eigenvalues.
\end{proof}

 \subsection{Additional Figures for Experiments} 

\subsubsection{Swiss Roll Experiment}

The first of these figures is the first two principal components of the feature map $\Lap^{-\frac{1}{2}}$ applied to the distributions, which demonstrates the ability of GFMMD to capture nonlinear directions in a linear space in the presence of strong noise.
On the left of Figure ~\ref{fig:swiss_roll} is EMD, where the oscillatory pattern illustrates its ineffectiveness at calculating distances between distributions on graphs, since Euclidean distance between points on the swiss roll has periodic behavior in curvature. Diffusion EMD and Kernel MMD are effective at taking distances 
    between points initially, but fail to discern between higher and higher distances. Graph Fourier MMD, on the other hand, has a far more clear linear correlation, which levels off much slower.  \\ 

\begin{figure}
    \centering
    \includegraphics[width=0.6\linewidth]{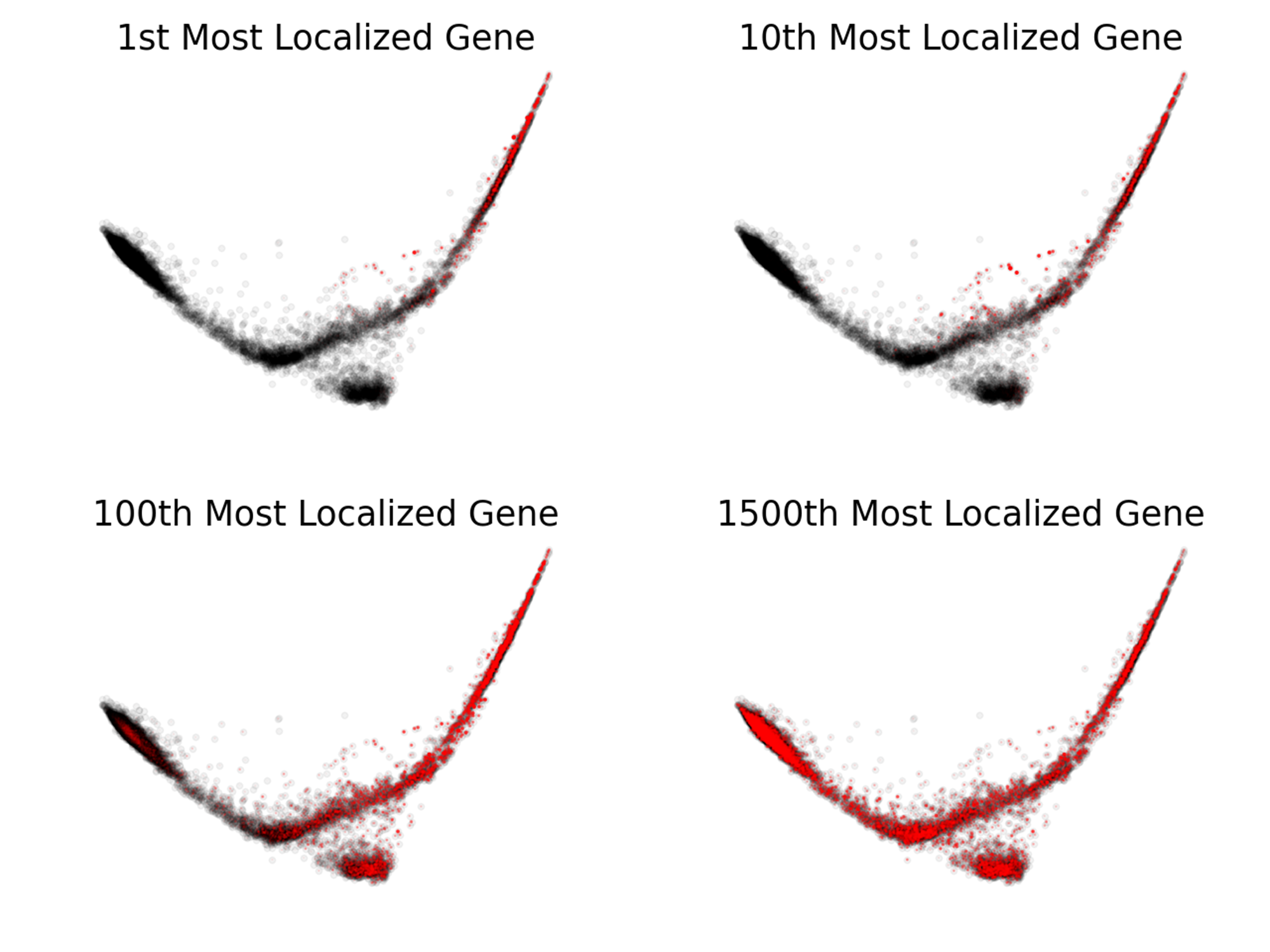}
    \includegraphics[width=0.39\linewidth]{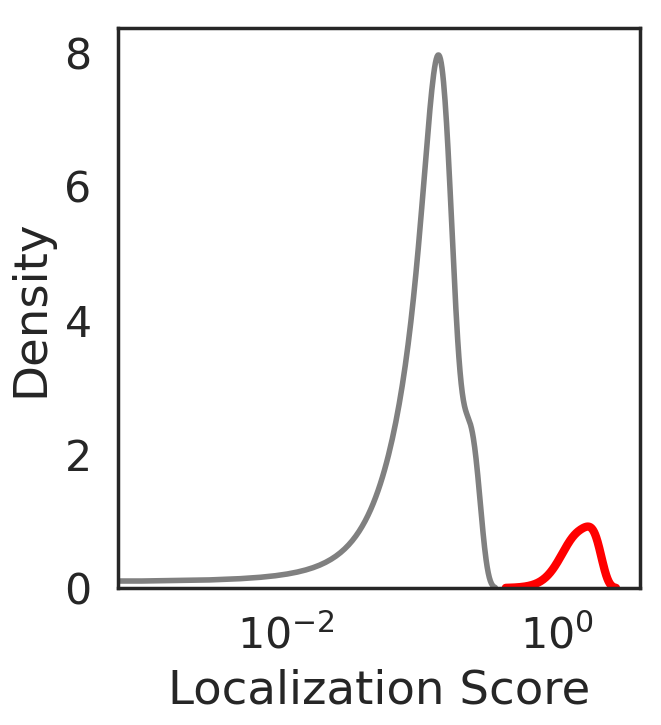}
    \caption{We visualize the 1st, 10th, 100th, and 1500th most local genes on the cell graph. Indeed, we find the expected behavior. Density plots for localization scores, comparing housekeeping genes and naive CD8+ T cell signature. The naive gene signatures are given by the red curve and Housekeeping gene signatures by the gray.}
    \label{fig:localization}
    \vspace{-8mm}
\end{figure} 

\subsubsection{Single Cell Localization}

Below, we have visualizations of the spread of the most localized signals over the graph. Here, PHATE is used to produce two dimensional embeddings of cells in Euclidean space, and color intensity is used as an indicator for gene expression. 

 \begin{figure}[H]
    \centering
    \includegraphics[scale=0.15]{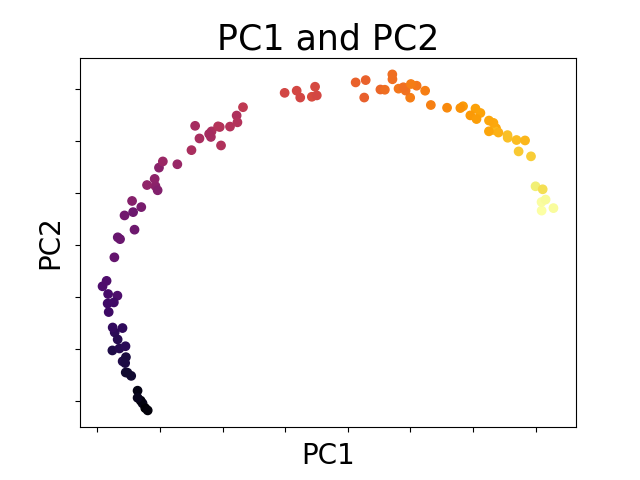}
    \includegraphics[scale=0.15]{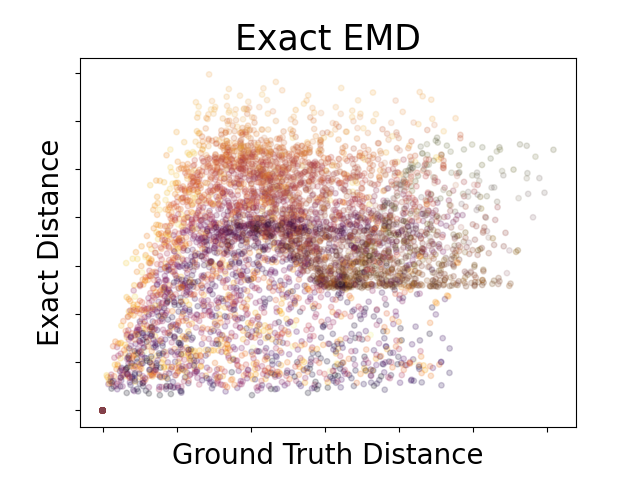}
    \includegraphics[scale=0.15]{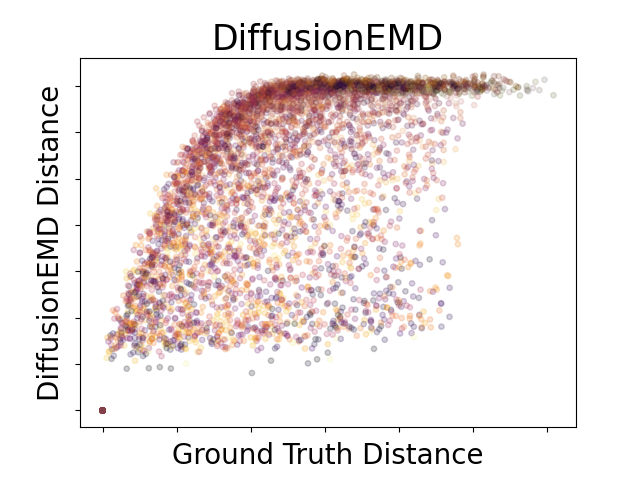}
    \includegraphics[scale=0.15]{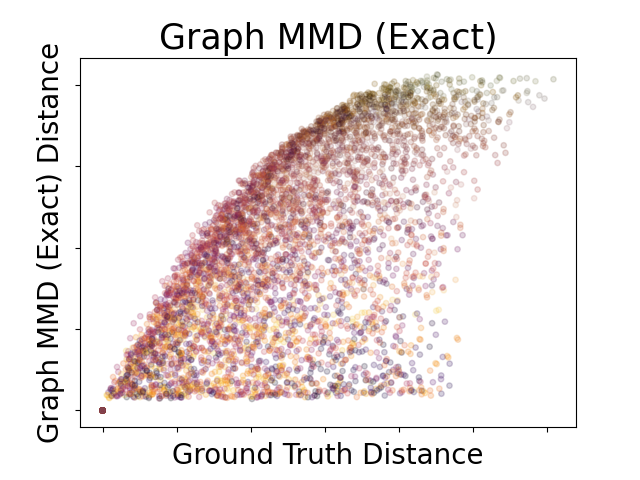}
    \includegraphics[scale=0.15]{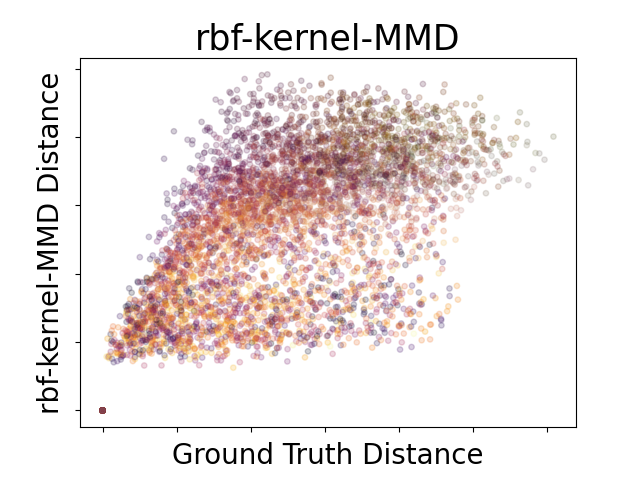}
    
    \caption{Left: first two PCs of the embeddings $E$ from Algorithm ~\ref{algorithm:GFMMDAlg}., colored by the coordinate of the corresponding center along the curved direction of the swiss roll. Right: Geodesic distance between centers vs. corresponding distance between distributions}
    \label{fig:swiss_roll}
    \vspace{-4mm}
\end{figure}

\subsection{Additional Toy Experiments}

\paragraph{Grid Graph} First, we consider a $16 \times 16$ grid graph (vertices given by $\{(i,j)\}_{1 \leq i,j \leq 16}$. We can construct a signal $P$ by placing a Dirac $\delta_{(8,4)}$ on the vertex (8,4) and then diffusing it with a heat filter (using time $\tau = 16$). $Q$ is generated likewise, but by applying a heat filter to $\delta_{(8, 4 + 2j)}$ and diffusing for each $j = 0,1,2,3$. The result are two modes: $P$ on the left, and $Q$ moving along the right. The distributions are visualized in the top row, and the witness function to their difference in the bottom row of figure ~\ref{fig:grid_graph}.

And of course, the corresponding distances between $P$ and the $Q$'s (per the order presented above) are increasing in the distances between the appropriate centers.

\begin{figure}[H]
    \centering
    \includegraphics[width=0.2\linewidth]{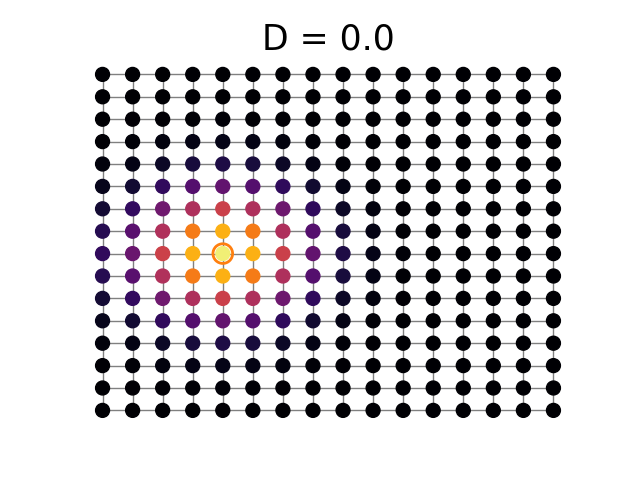}
    \includegraphics[width=0.2\linewidth]{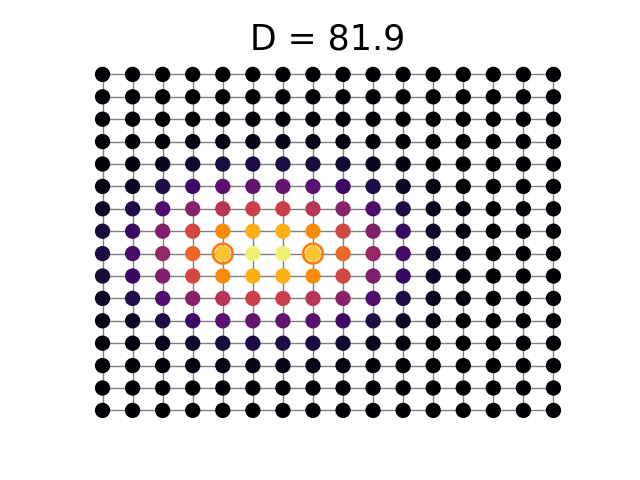}
    \includegraphics[width=0.2\linewidth]{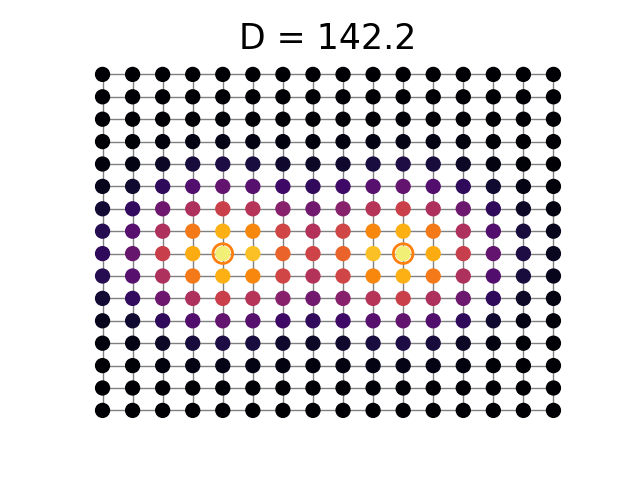}
    \includegraphics[width=0.2\linewidth]{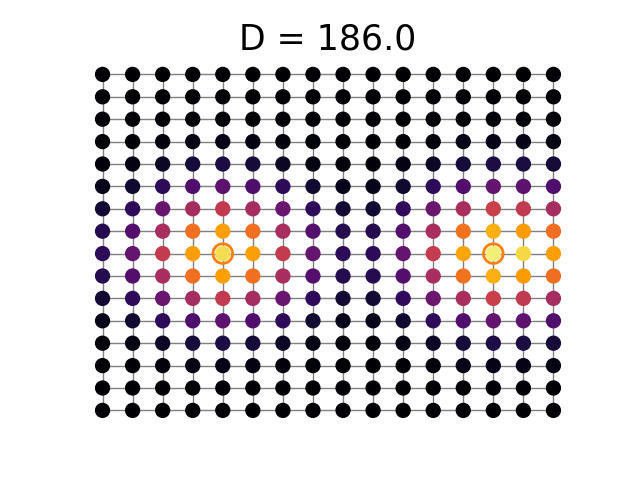}
    \includegraphics[width=0.2\linewidth]{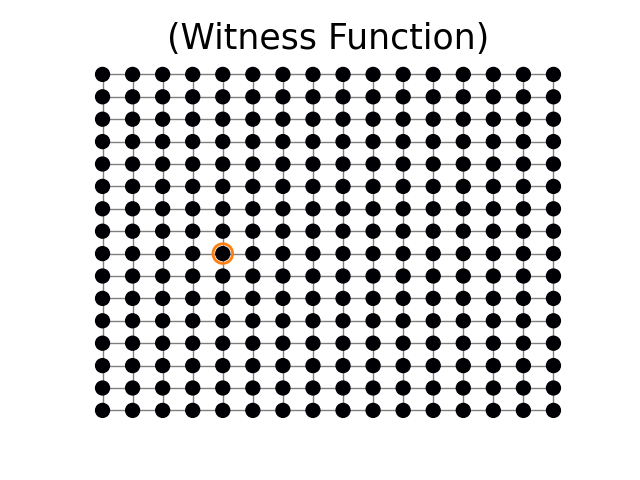}
    \includegraphics[width=0.2\linewidth]{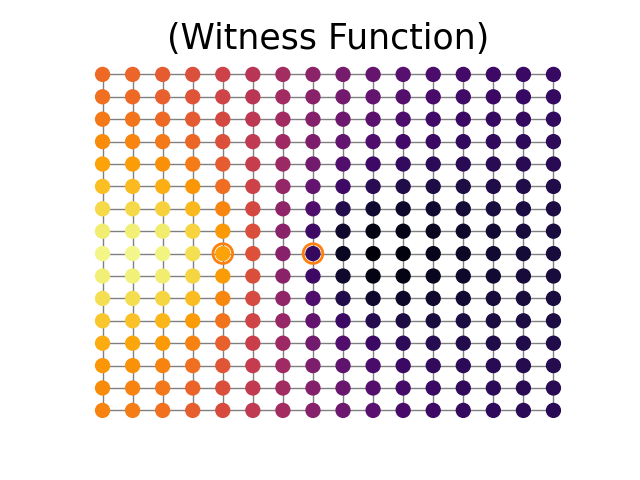}
    \includegraphics[width=0.2\linewidth]{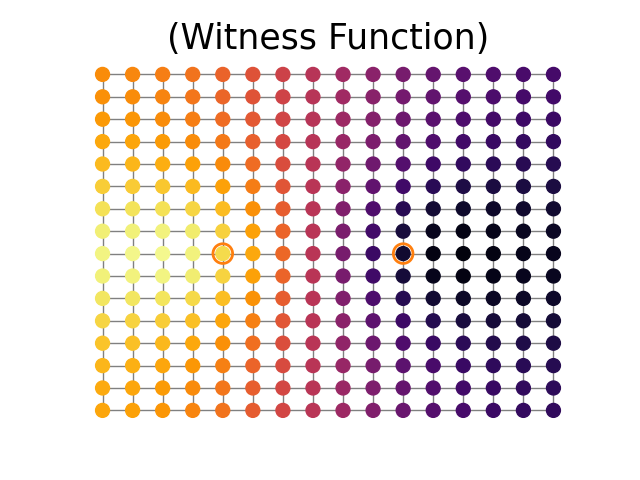}
    \includegraphics[width=0.2\linewidth]{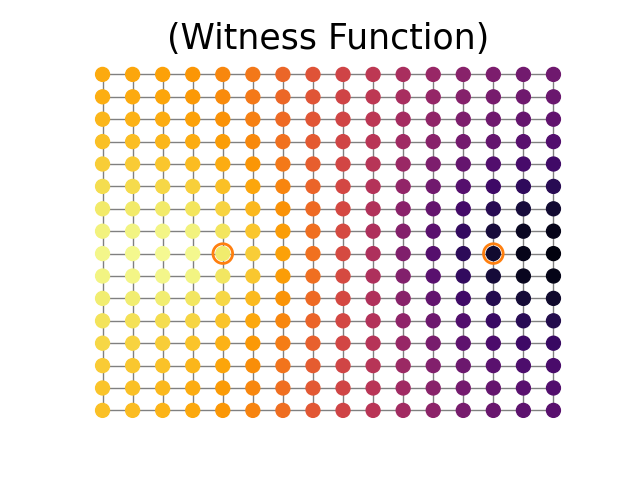}
    \caption{Top row: the distributions $P$ and $Q$, where the signal $P$ stays fixed but the vertex at which $Q$ is centered shifts to the right. Corresponding distances between distributions appear in the title, and the relevant centers of $P$ and $Q$ are highlighted. Bottom row: corresponding witness functions $f$ to the difference between $P$ and $Q$.}
    \label{fig:grid_graph}

\end{figure}

\paragraph{Bunny Graph} One very simple sanity check of a measure of spread is to verify that the more we diffuse a Dirac, the lower the distance to the uniform. Indeed, if we begin with the Bunny graph (from pygsp's built in library) and diffuse the Dirac $\delta_{1400}$ ($1400$ was chosen for visual appeal) for scales $\tau = 2^0, 2^4, 2^8$, and $2^12$ (using a heat filter), we find that the corresponding measures of spread are 40.5, 26.9, 21.5, and 9.76. The signals are visualized below:

\begin{figure}[H]
    \centering
    \includegraphics[scale=0.2]{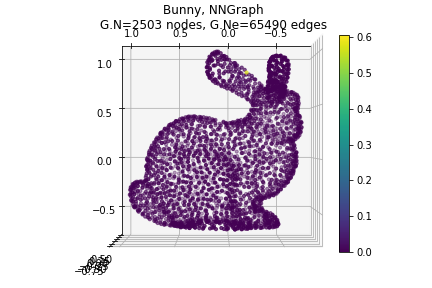}
    \includegraphics[scale=0.2]{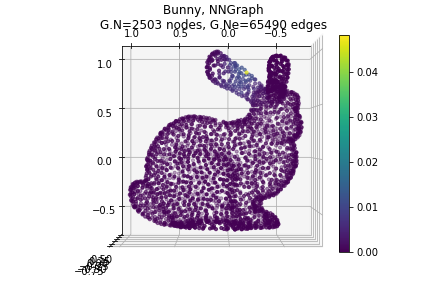}
    \includegraphics[scale=0.2]{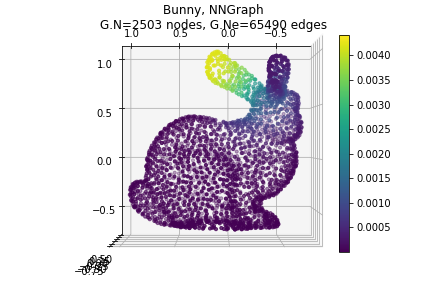}
    \includegraphics[scale=0.2]{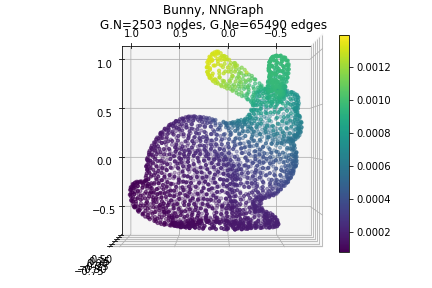}
    \caption{The signal $\delta_{1400}$ diffused to levels 1,$6^1$, $6^2$, and $6^3$ using a heat filter.}
    \label{fig:bunny}
\end{figure}

\subsection{ Bimodal Signals}

We can take the earlier signals from the grid graph (each pair of $P$ and $Q$ for translations of $Q$) and combine them into a new signal $\frac{1}{2}(P + Q)$. This forms a family of bimodal signals for which the two modes spread. Accordingly, in the example above, the distance to the uniform is given by 11.14, 8.66, 6.13, and 6.09.

\subsection{Localization on the Minnesota Graph}

\subsubsection{Example: Minnesota Graph (Binarized)} 

A final sanity check for a measure of closeness to the uniform would be to begin with a density which puts all its mass on one vertex. Then, put equal mass on that vertex and its neighbors, then the neighbors of neighbors, etc. More specifically, let $N_k(i,j) = \{ \exists k' \in [k] : A^{k'} > 0\}$, or $N_k(i,j) = \mathbf{1}\{$there is a path of length $\leq k$ from $i$ to $j\}$. Then we can consider multiplying this by a Dirac, say $\delta_0$ to get a family of signals. Using $k = 1, 4^1, 4^2, 4^3$, we have a family of distributions proportional to $N_1\delta_0, N_2 \delta_0, N_3\delta_0$, and $N_4\delta_0$. Again, we can visualize the activated vertices in yellow:

\begin{figure}[H]
    \centering
    \includegraphics[width=0.3\linewidth]{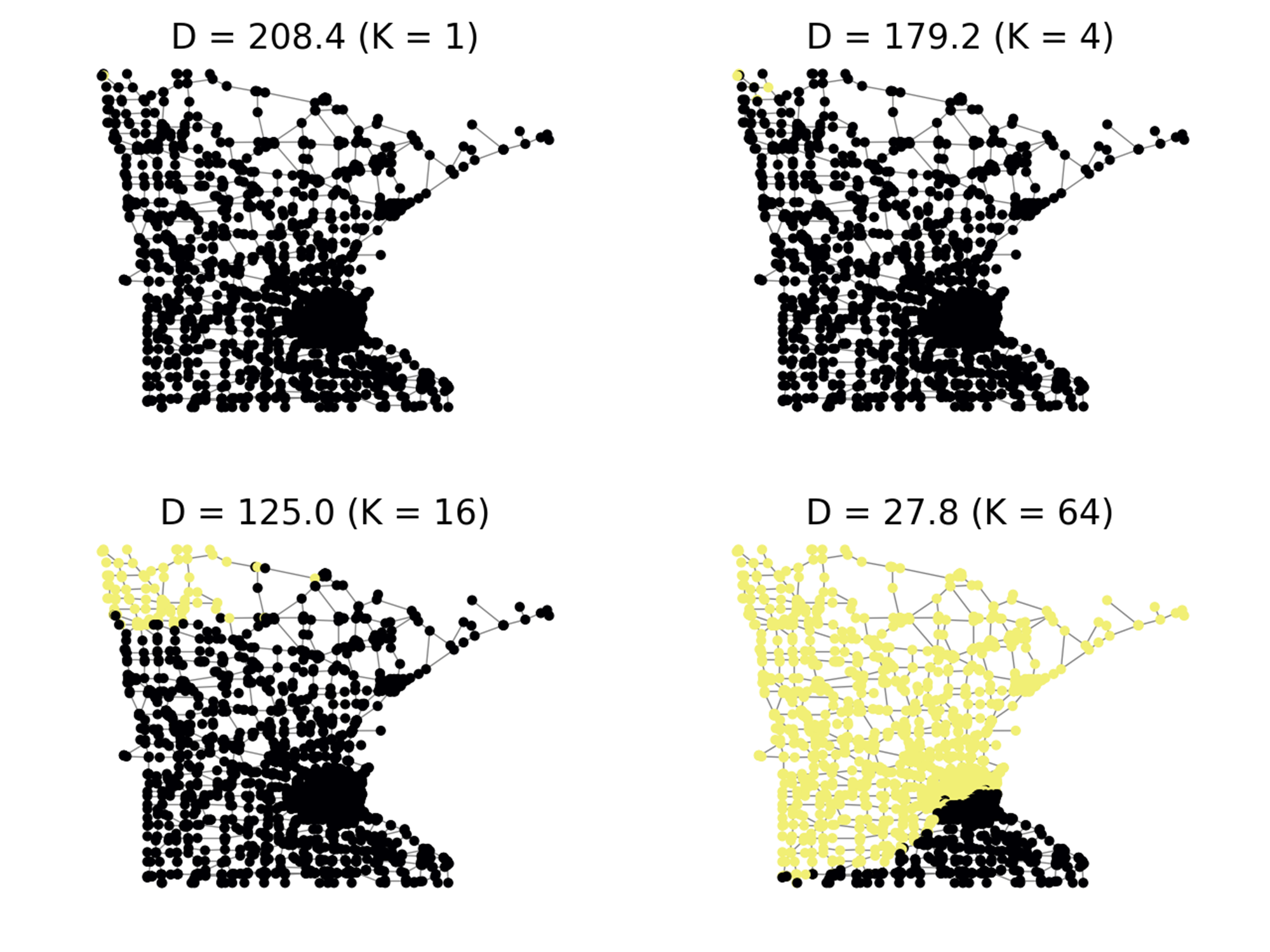}
    \caption{The zeroth vertex's neighbors, then neighbors of neighbors, etc.\ for order $1, 4, 16$, and $64$ neighbors. The corresponding distances to the uniform are given in the title. }
    \label{fig: meta}
\end{figure}

\subsection{Example: Minnesota Graph (Smooth Waves)}

A similar example we can consider is a similar class of signals which "spread" across the graph, but rather than activating neighbors, simply diffusing the signal from a given start vertex. Here, we choose the same start vertex, and run heat diffusion at times $\tau = 2^0,2^4, 2^8$, and $2^{12}$. 

\begin{figure}[H]
    \centering
    \includegraphics[width=0.3\linewidth]{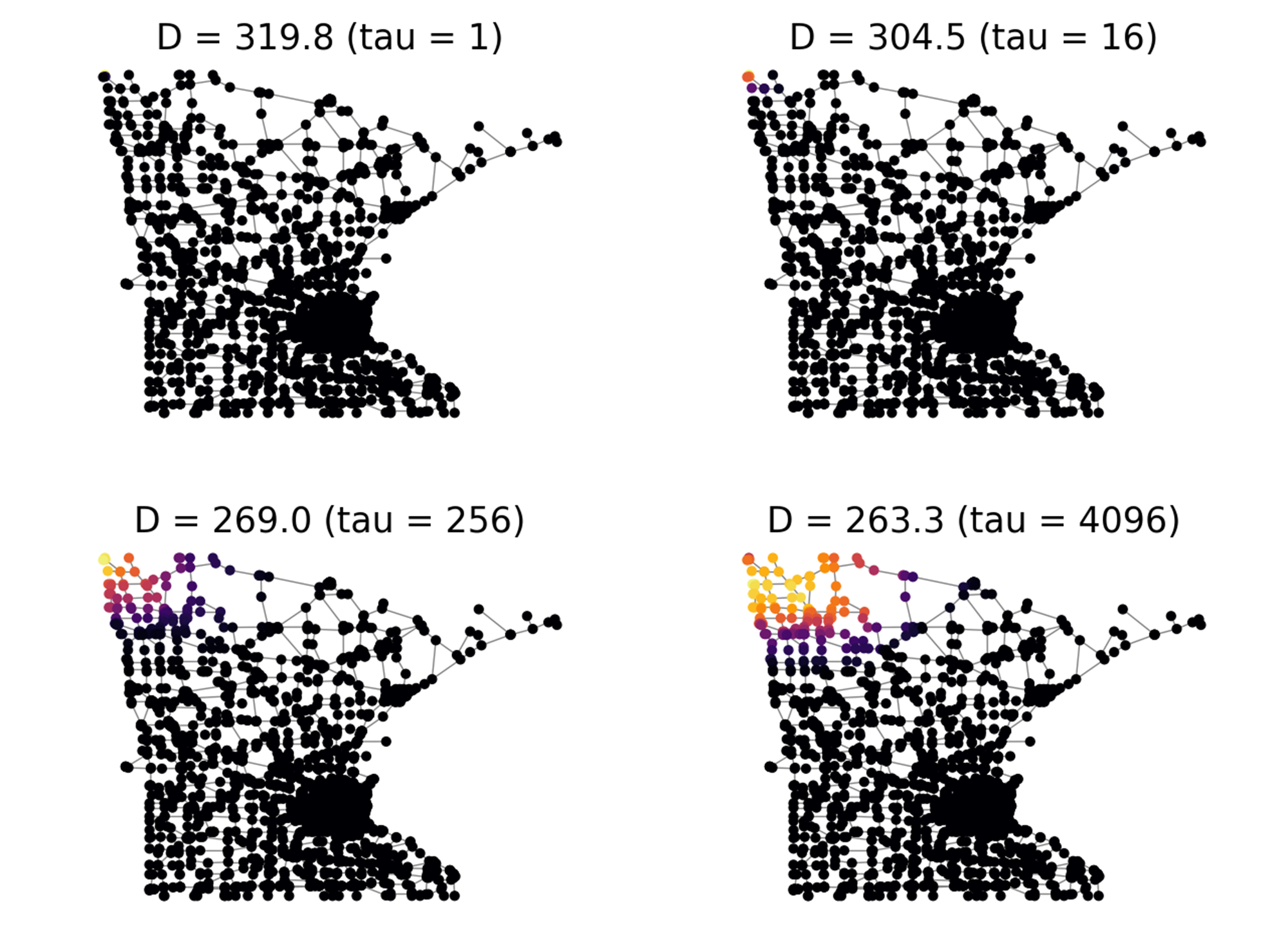}
    \caption{Visualization of the diffusions of the signal $\delta_0$ at times $2^0, 2^4, 2^8$, and $2^{12}$. The corresponding distances to the uniform are given above.}
    \label{fig: meta_diffused}
\end{figure}

\end{document}